\documentclass[journal]{IEEEtran}
\usepackage[linesnumbered,ruled]{algorithm2e}
\usepackage{graphicx,amssymb,mathrsfs,amsmath,array,color,amsthm}
\usepackage{subfigure}
\usepackage{multirow}
\usepackage{booktabs}

\usepackage{mathtools}
\usepackage{bbm}
\usepackage{cite}
\usepackage[colorlinks,urlcolor=blue,driverfallback=dvipdfm]{hyperref}
\newtheorem{lemma}{Lemma}
\newtheorem{theorem}{Theorem}
\newtheorem{proposition}{Proposition}

\newcommand{\norm}[1]{\lVert#1\rVert}

\DeclarePairedDelimiter{\ceil}{\lceil}{\rceil}

\begin{document}
\title{Graph Convolutional Networks for Hyperspectral Image Classification}

\author{Danfeng Hong,~\IEEEmembership{Member,~IEEE,}
        Lianru Gao,~\IEEEmembership{Senior Member,~IEEE,}
        Jing Yao,
        Bing Zhang,~\IEEEmembership{Fellow,~IEEE,}\\
        Antonio Plaza,~\IEEEmembership{Fellow,~IEEE,}
        and Jocelyn Chanussot,~\IEEEmembership{Fellow,~IEEE}

\thanks{This work was supported by the National Natural Science Foundation of China under Grant 41722108 and Grant 91638201 as well as with the support of the AXA Research Fund. (\emph{Corresponding author: Lianru Gao})}
\thanks{D. Hong is with the Univ. Grenoble Alpes, CNRS, Grenoble INP, GIPSA-lab, 38000 Grenoble, France. (e-mail: hongdanfeng1989@gmail.com)}
\thanks{L. Gao is with the Key Laboratory of Digital Earth Science, Aerospace Information Research Institute, Chinese Academy of Sciences, Beijing 100094, China. (e-mail: gaolr@aircas.ac.cn)}
\thanks{J. Yao is with the School of Mathematics and Statistics, Xi’an Jiaotong University, 710049 Xi’an, China. (e-mail: jasonyao@stu.xjtu.edu.cn)}
\thanks{B. Zhang is with the Key Laboratory of Digital Earth Science, Aerospace Information Research Institute, Chinese Academy of Sciences, Beijing 100094, China, and also with the College of Resources and Environment, University of Chinese Academy of Sciences, Beijing 100049, China. (e-mail: zb@radi.ac.cn)}
\thanks{A. Plaza is with the Hyperspectral Computing Laboratory, Department of Technology of Computers and Communications, Escuela Polit\'ecnica, University of Extremadura, 10003 C\'aceres, Spain. (e-mail: \mbox{aplaza@unex.es}).}
\thanks{J. Chanussot is with the Univ. Grenoble Alpes, INRIA, CNRS, Grenoble INP, LJK, F-38000 Grenoble, France, and also with the Aerospace Information Research Institute, Chinese Academy of Sciences, Beijing 100094, China. (e-mail: jocelyn@hi.is)}
}

\markboth{Submission to IEEE Transactions on Geoscience and Remote Sensing,~Vol.~XX, No.~XX, ~XXXX,~2020}
{Shell \MakeLowercase{\textit{et al.}}: Graph Convolutional Networks for Hyperspectral Image Classification}

\maketitle
\begin{abstract}
\textcolor{blue}{This is the pre-acceptance version, to read the final version please go to IEEE Transactions on Geoscience and Remote Sensing on IEEE Xplore.} Convolutional neural networks (CNNs) have been attracting increasing attention in hyperspectral (HS) image classification, owing to their ability to capture spatial-spectral feature representations. Nevertheless, their ability in modeling relations between samples remains limited. Beyond the limitations of grid sampling, graph convolutional networks (GCNs) have been recently proposed and successfully applied in irregular (or non-grid) data representation and analysis. In this paper, we thoroughly investigate CNNs and GCNs (qualitatively and quantitatively) in terms of HS image classification. Due to the construction of the adjacency matrix on all the data, traditional GCNs usually suffer from a huge computational cost, particularly in large-scale remote sensing (RS) problems. To this end, we develop a new mini-batch GCN (called miniGCN hereinafter) which allows to train large-scale GCNs in a mini-batch fashion. More significantly, our miniGCN is capable of inferring out-of-sample data without re-training networks and improving classification performance. Furthermore, as CNNs and GCNs can extract different types of HS features, an intuitive solution to break the performance bottleneck of a single model is to fuse them. Since miniGCNs can perform batch-wise network training (enabling the combination of CNNs and GCNs) we explore three fusion strategies: additive fusion, element-wise multiplicative fusion, and concatenation fusion to measure the obtained performance gain. Extensive experiments, conducted on three HS datasets, demonstrate the advantages of miniGCNs over GCNs and the superiority of the tested fusion strategies with regards to the single CNN or GCN models. The codes of this work will be available at \url{https://github.com/danfenghong/IEEE_TGRS_GCN} for the sake of reproducibility.
\end{abstract}
\graphicspath{{figures/}}


\section{Introduction}
\IEEEPARstart{L}{and} use and land cover (LULC) classification \cite{anderson1976land} using Earth observation (EO) data, e.g., hyperspectral (HS) \cite{rasti2020feature}, synthetic aperture radar (SAR) \cite{kang2020learning}, light detection and ranging (LiDAR) \cite{huang2019multi}, etc. is a challenging topic in geoscience and remote sensing (RS). Characterized by their rich and detailed spectral information, HS images allow discriminating the objects of interest more effectively (particularly those in spectrally similar classes) by capturing more subtle discrepancies from the contiguous shape of the spectral signatures associated to their pixels. HS imagery enables the detection and recognition of the materials on the Earth's surface at a more fine and accurate level compared to RGB and multispectral (MS) data. However, the high spectral mixing between materials \cite{bioucas2012hyperspectral} and spectral variability and complex noise effects \cite{hong2019augmented} bring difficulties in extracting discriminative information from such data.

Over the past decades, a variety of hand-crafted and learning-based feature extraction (FE) algorithms \cite{ghamisi2017advances,peng2017robust,peng2018self,liu2019unsupervised,hong2019learning,liu2019review,wang2019spatial,samat2020meta,hong2020x} (either unsupervised or supervised) have been successfully designed for HS image classification. Among them, morphological profiles (MPs) \cite{benediktsson2005classification} are an effective tool that allows us to manually extract spatial-spectral features from HS images. In \cite{fauvel2008spectral}, Fauvel \textit{et al.} used MPs as input vectors for a support vector machine (SVM) classifier, achieving satisfactory classification results. Samat \textit{et al.} \cite{samat2018classification} designed new maximally stable extremal region-guided MPs, yielding a high classification performance on multispectral images. Other works based on morphological operations have been developed to further enhance feature representations, including attribute profiles (APs) \cite{dalla2010morphological}, invariant APs \cite{hong2020invariant,wu2020fourier}. Another typical FE strategy is subspace-based learning, e.g., sparse representation \cite{chen2011hyperspectral,gao2020spectral} and manifold learning \cite{hong2017learning,hong2019learning}. These methods learn transformations or projections by managing the high-dimensional original space using a new, latent, low-dimensional subspace representation. Although the aforementioned approaches have been proven to be effective in HS classification tasks, feature discrimination still remains limited due to the lack of powerful data fitting and representation ability.

Inspired by the success of deep learning (DL) techniques, significant progress has been made in the area of HS image classification by using various advanced deep networks \cite{li2019deep}. Chen \textit{et al.} \cite{chen2014deep} applied stacked auto-encoder networks to dimensionally-reduced HS images --obtained by principal component analysis (PCA) -- for HS image classification. Further, Chen \textit{et al.} \cite{chen2016deep} adopted convolutional neural networks (CNNs) to extract spatial-spectral features more effectively from HS images, thereby yielding higher classification performance. Recurrent neural networks (RNNs) \cite{liu2017bidirectional,wu2017convolutional} can process the spectral signatures as sequential data. In \cite{hang2019cascaded}, a cascaded RNN was proposed to make full use of spectral information for high-accuracy HS image classification. Recently, Hang \textit{et al.} \cite{hang2020classification} developed a multitask generative adversarial networks and provided new insight into HS image classification, yielding state-of-the-art performance.

\begin{figure}[!t]
	  \centering
		\subfigure{
			\includegraphics[width=0.48\textwidth]{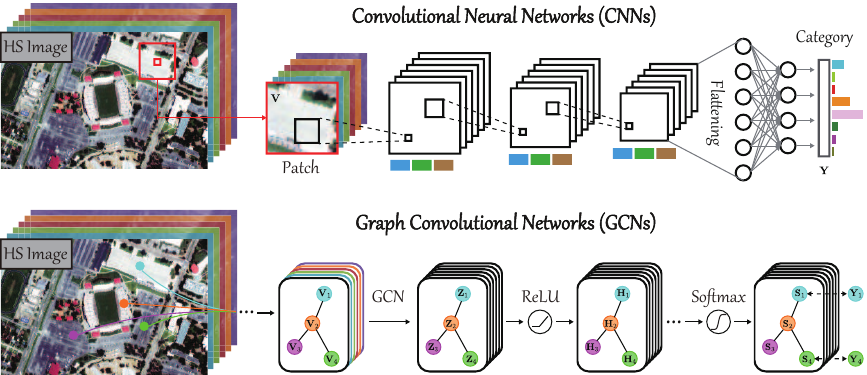}
		}
        \caption{Comparison of CNN and GCN architectures in HS image classification tasks. The variables of $\mathbf{V}$, $\mathbf{Z}$, $\mathbf{H}$, $\mathbf{S}$, and $\mathbf{Y}$ in GCNs denote vertexes, hidden representations via GCN layer, hidden representations via ReLU layer, hidden representations via softmax layer, and labels, respectively.}
\label{fig:CNNs_GCNs}
\end{figure}

Comparatively, graph convolutional networks (GCNs) \cite{kipf2016semi} are a hot topic and emerging network architecture, which is able to effectively handle graph structure data by modeling relations between samples (or vertexes). Therefore, GCNs can be naturally used to model long-range spatial relations in the HS image (see Fig. \ref{fig:CNNs_GCNs}), which fail to be considered in CNNs. Currently, GCNs are less popular than CNNs in HS image classification. There are a few works related to the use of GCNs in HSI classification, though. Shahraki \textit{et al.} \cite{shahraki2018graph} proposed to cascade 1-D CNNs and GCNs for HS image classification. Qin \textit{et al.} \cite{qin2018spectral} extended the original GCNs to a second-order version by simultaneously considering spatial and spectral neighborhoods. Wan \textit{et al.} \cite{wan2019hyperspectral} performed superpixel segmentation on the HS image and fed it into GCN to reduce the computational cost and improve the classification accuracy. However, there are some potential limitations of GCNs regarding the following aspects: 
\begin{itemize}
    \item The high computational cost (resulting from the construction of the adjacency matrix) is a significant bottleneck of GCNs in the HS image classification task, particularly when using large-scale HS image data.
    \item GCNs only allow for full-batch network learning, that is, feeding all samples at once into the network. This might lead to large memory costs and slow gradient descent, as well as the negative effects of variable updating.
    \item Last but not least, a trained GCN-based model fails to predict the new input samples (i.e. out-of-samples) without re-training the network, which has a major influence on the use of GCNs in practice.
\end{itemize}

To overcome these difficulties, in this work we introduce a simple but effective mini-batch GCN (called miniGCN). Similar to CNNs, miniGCNs can effectively train the network for classification on a downsampled graph (or topological structure) in mini-batch fashion, and meanwhile the learned model can be directly used for prediction purposes on new data. In addition, with our newly proposed miniGCNs, we aim to make a side-by-side comparison between CNNs and GCNs (both qualitatively and quantitatively) and raise an interesting question: \textit{which one between CNNs and GCNs can assist more in the HS image classification task?} It is well known that CNNs and GCNs can extract and represent spectral information from HS images using different perspectives, i.e., spatial-spectral features of CNNs, graph (or relation) representations of GCNs, etc. This naturally motivates us to jointly use them by investigating different fusion strategies, making them even more suitable for HS image classification. More specifically, the main contributions of this paper are three-fold:

\begin{itemize}
    \item We systematically analyze CNNs and GCNs with a focus on HS image classification. To the best of our knowledge, this is the first time that the potentials and drawbacks of GCNs (in comparison with CNNs) are investigated the community.
    \item We propose a novel supervised version of GCNs: miniGCNs, for short. As the name suggests, miniGCNs can be trained in mini-batch fashion, trying to find a better and more robust local optimum. Unlike traditional GCNs, our miniGCNs are not only capable of training the networks using training set, but also allow for a straightforward inference of large-scale, out-of-samples using the trained model.
    \item We develop three fusion schemes, including additive fusion, element-wise multiplicative fusion, and concatenation fusion, to achieve better classification results in HS images by integrating features extracted from CNNs and our miniGCNs, in an end-to-end trainable network.
\end{itemize}

The remaining of the paper is organized as follows. Section II deeply reviews GCN-related knowledge. Section III elaborates on the proposed miniGCNs and introduces three different fusion strategies in the context of a general end-to-end fusion network. Extensive experiments and analyses are given in section IV. Section V concludes the paper with some remarks and hints at plausible future research work.

\section{Review of GCNs}
In this section, we provide some preliminaries of GCNs by reviewing the basic definitions and notations, including graph construction and several important theorems and proofs for graph convolution in the spectral domain.

\subsection{Definition of Graph} 
A graph is a complex nonlinear data structure, which is used to describe a one-to-many relationship in a non-Euclidean space. In our case, the relations between spectral signatures can be represented as an undirected graph. Let $\mathcal{G} = (\mathcal{V}, \mathcal{E})$ be an undirected graph, where $\mathcal{V}$ and $\mathcal{E}$ denote the vertex and edge sets, respectively. In our context, the vertex set consists of HS pixels, while the edge set is composed of the similarities between any two vertexes, i.e., $\mathcal{V}_i$ and $\mathcal{V}_j$.

\subsection{Construction of the Adjacency Matrix} 
The adjacency matrix, denoted as $\mathbf{A}$, defines the relationships (or edges) between vertexes. Each element in $\mathbf{A}$ can be generally computed by using the following radial basis function (RBF):
\begin{equation}
\label{eq1}
\begin{aligned}
      \mathbf{A}_{i,j} = \mathrm{exp}(-\frac{\norm{\mathbf{x}_{i}-\mathbf{x}_{j}}^{2}}{\sigma^{2}}),
\end{aligned}
\end{equation}
where $\sigma$ is a parameter to control the width of the RBF. The vectors $\mathbf{x}_{i}$ and $\mathbf{x}_{j}$ denote the spectral signatures associated to the vertexes $v_i$ and $v_j$. Once $\mathbf{A}$ is given, we create the corresponding graph Laplacian matrix $\mathbf{L}$ as follows:
\begin{equation}
\label{eq2}
\begin{aligned}
      \mathbf{L}=\mathbf{D}-\mathbf{A},
\end{aligned}
\end{equation}
where $\mathbf{D}$ is a diagonal matrix representing the degrees of $\mathbf{A}$, i.e., $\mathbf{D}_{i,i}=\sum_{j}\mathbf{A}_{i,j}$ \cite{hong2019cospace,hong2019learnable}. To enhance the generalization ability of the graph \cite{chung1997spectral}, the symmetric normalized Laplacian matrix ($\mathbf{L}_{sym}$) can be used as follows:
\begin{equation}
\label{eq3}
\begin{aligned}
      \mathbf{L}_{sym}&=\mathbf{D}^{-\frac{1}{2}}\mathbf{L}\mathbf{D}^{-\frac{1}{2}}\\
      &=\mathbf{I}-\mathbf{D}^{-\frac{1}{2}}\mathbf{A}\mathbf{D}^{-\frac{1}{2}},
\end{aligned}
\end{equation}
where $\mathbf{I}$ is the identity matrix.

\subsection{Graph Convolutions in the Spectral Domain} 
Given two functions $f$ and $g$, their convolution can be then written as:
\begin{equation}
\label{eq4}
\begin{aligned}
      f(t)\star g(t)\triangleq\int_{-\infty }^{\infty}f(\tau)g(t-\tau)d\tau,
\end{aligned}
\end{equation}
where $\tau$ is the shifting distance and $\star$ denotes the convolution operator. 

\begin{theorem}
The Fourier transform of the convolution of two functions $f$ and $g$ is the product of their corresponding Fourier transforms. This can be formulated as 
\begin{equation}
\label{eq5}
\begin{aligned}
      \mathcal{F}[f(t)\star g(t)]=\mathcal{F}[f(t)]\cdot\mathcal{F}[g(t)],
\end{aligned}
\end{equation}
where $\mathcal{F}$ and $\cdot$ denote the Fourier transform and point-wise multiplication, respectively.
\end{theorem}

\begin{theorem}
The inverse Fourier transform ($\mathcal{F}^{-1}$) of the convolution of two functions $f$ and $g$ is equal to $2\pi$ the product of their corresponding inverse Fourier transforms:
\begin{equation}
\label{eq6}
\begin{aligned}
      \mathcal{F}^{-1}[f(t)\star g(t)]=2\pi\mathcal{F}^{-1}[f(t)]\cdot\mathcal{F}^{-1}[g(t)].
\end{aligned}
\end{equation}
\end{theorem}

By means of the above two well-known Theorems \cite{mcgillem1991continuous}, i.e., Eqs. (\ref{eq5}) and (\ref{eq6}), the convolution can be generalized to the graph signal as:
\begin{equation}
\label{eq7}
\begin{aligned}
      f(t)\star g(t)=\mathcal{F}^{-1}\{\mathcal{F}[f(t)]\cdot\mathcal{F}[g(t)]\}].
\end{aligned}
\end{equation}
Hence, the convolution operation on a graph can be converted to define the Fourier transform ($\mathcal{F}$) or to find a set of basis functions.

\begin{lemma}
The basis functions of $\mathcal{F}$ can be equivalently represented by a set of eigenvectors of $\mathbf{L}$.
\end{lemma}

\begin{proof}
By referring to \cite{mcgillem1991continuous}, we have the following proof. For many functions that do not converge in domain, e.g., $y(t)=t^{2}$, we can always find a real-valued exponential function $e^{-\sigma t}$ to make $y(t)e^{-\sigma t}$ converge, thereby satisfying the Dirichlet condition of $\mathcal{F}$, i.e.,
\begin{equation}
\label{eq8}
\begin{aligned}
      \int_{-\infty }^{\infty}|y(t)e^{-\sigma t}|dt < \infty.
\end{aligned}
\end{equation}
Plugging $y(t)e^{-\sigma t}$ into $\mathcal{F}$, we have: 
\begin{equation}
\label{eq9}
\begin{aligned}
      \int_{-\infty }^{\infty}y(t)e^{-\sigma t}e^{-2\pi ix\xi}dt,
\end{aligned}
\end{equation}
and we can rewrite Eq. (\ref{eq9}) as:
\begin{equation}
\label{eq10}
\begin{aligned}
      \int_{-\infty }^{\infty}y(t)e^{-st}dt,
\end{aligned}
\end{equation}
where $s=\sigma+2\pi ix\xi$. Note that Eq. (\ref{eq10}) is the Laplace transform. In other words, the eigenvectors of $\mathbf{L}$ are identical to the basis functions of $\mathcal{F}$.
\end{proof}

Given \textbf{Lemma 1}, we can perform spectral decomposition on $\mathbf{L}$. We then have:
\begin{equation}
\label{eq11}
\begin{aligned}
      \mathbf{L}=\mathbf{U}\mathbf{\Lambda}\mathbf{U}^{-1},
\end{aligned}
\end{equation}
where $\mathbf{U}=(\mathbf{u}_{1}, \mathbf{u}_{2},\dots, \mathbf{u}_{n})$ is the set of eigenvectors of $\mathbf{L}$, that is, the basis of $\mathcal{F}$. As $\mathbf{U}$ is the orthogonal matrix, i.e., $\mathbf{U}\mathbf{U}^{\top}=\mathbf{E}$, Eq. (\ref{eq11}) can be also written as:
\begin{equation}
\label{eq12}
\begin{aligned}
      \mathbf{L}=\mathbf{U}\mathbf{\Lambda}\mathbf{U}^{-1}=\mathbf{U}\mathbf{\Lambda}\mathbf{U}^{\top}.
\end{aligned}
\end{equation}
According to Eq. (\ref{eq12}), $\mathcal{F}$ of $f$ on a graph is $\mathcal{GF}[f]=\mathbf{U}^{\top}f$, and the inverse transform becomes $f=\mathbf{U}\mathcal{GF}[f]$. In analogy with Eq. (\ref{eq7}), the convolution between $f$ and $g$ on a graph can be expressed as:
\begin{equation}
\label{eq13}
\begin{aligned}
      \mathcal{G}[f\star g]=\mathbf{U}\{[\mathbf{U}^{\top}f]\cdot[\mathbf{U}^{\top}g]\}.
\end{aligned}
\end{equation}
If we write $\mathbf{U}^{\top}g$ as $g_{\theta}$, the convolution on a graph can be finally formulated as: 
\begin{equation}
\label{eq14}
\begin{aligned}
      \mathcal{G}[f\star g_{\theta}]=\mathbf{U}g_{\theta}\mathbf{U}^{\top}f,
\end{aligned}
\end{equation}
where $g_{\theta}$ can be regarded as the function of the eigenvalues ($\mathbf{\Lambda}$) of $\mathbf{L}$ with the respect to the variable $\theta$, i.e., $g_{\theta}(\mathbf{\Lambda})$.

To reduce the computational complexity of Eq. (\ref{eq14}), Hammond \textit{et al.} \cite{hammond2011wavelets} approximately fitted $g_{\theta}$ by applying the $K$-th order truncated expansion of Chebyshev polynomials. By doing so, Eq. (\ref{eq14}) can be rewritten as: 
\begin{equation}
\label{eq15}
\begin{aligned}
      \mathcal{G}[f\star g_{\theta}]\approx \sum_{k=0}^{K}\theta_{k}^{'}T_{k}(\mathbf{\widetilde{L}})f,
\end{aligned}
\end{equation}
where $T_{k}(\bullet)$ and $\theta_{k}^{'}$ are the Chebyshev polynomials with respect to the variable $\bullet$ and the Chebyshev coefficients, respectively. $\mathbf{\widetilde{L}}=\frac{2}{\lambda_{\max}}\mathbf{L}_{sym}-\mathbf{I}$ denotes the normalized $\mathbf{L}$.

\begin{figure*}[!t]
	  \centering
		\subfigure{
			\includegraphics[width=0.7\textwidth]{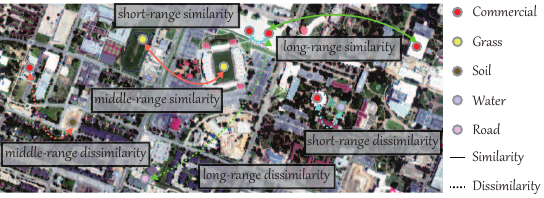}
		}
        \caption{An illustration of short-range, middle-range, and long-range spatial relations in a HS image. CNNs tend to extract locally spatial information, while GCNs are capable of capturing middle-range or long-range spatial relationships (either similarities or dissimilarities) between samples.}
\label{fig:spatial_relation}
\end{figure*}

By limiting $K=1$ and assigning the largest eigenvalue $\lambda_{\max}$ of $\mathbf{\widetilde{L}}$ to 2 \cite{kipf2016semi}, Eq. (\ref{eq15}) can be further simplified to: 
\begin{equation}
\label{eq16}
\begin{aligned}
      \mathcal{G}[f\star g_{\theta}]\approx \mathbf{\theta}(\mathbf{I}+\mathbf{D}^{-\frac{1}{2}}\mathbf{A}\mathbf{D}^{-\frac{1}{2}})f.
\end{aligned}
\end{equation}

Using Eq. (\ref{eq16}) we have the following propagation rule for GCNs:
\begin{equation}
\label{eq17}
\begin{aligned}
      \mathbf{H}^{(\ell+1)}=h(\mathbf{\widetilde{D}}^{-\frac{1}{2}}\mathbf{\widetilde{A}}\mathbf{\widetilde{D}}^{-\frac{1}{2}}\mathbf{H}^{(\ell)}\mathbf{W}^{(\ell)}+\mathbf{b}^{(\ell)}),
\end{aligned}
\end{equation}
where $\mathbf{\widetilde{\mathbf{A}}}=\mathbf{A}+\mathbf{I}$ and $\mathbf{\widetilde{D}}_{i,i}=\sum_{j}\mathbf{\widetilde{\mathbf{A}}}_{i,j}$ are defined as the renormalization terms of $\mathbf{A}$ and $\mathbf{D}$, respectively, to enhance stability in the process of network training. Moreover, $\mathbf{H}^{(\ell)}$ denotes the output in the $\ell^{th}$ layer and $h(\bullet)$ is the activation function (e.g., ReLU, used in our case) with respect to the weights to-be-learned  $\{\mathbf{W}^{(\ell)}\}_{\ell=1}^{p}$ and the biases $\{\mathbf{b}^{(\ell)}\}_{\ell=1}^{p}$ of all layers ($\ell=1,2,\dots,p$).

\section{Methodology}
In this section, we systematically analyze CNNs and GCNs from four different perspectives and further develop an improvement of existing GCNs called miniGCNs, making them better applicable to the HS image classification task. Finally, we introduce three different fusion strategies, yielding a general end-to-end fusion network.

\subsection{CNNs versus GCNs: Qualitative Comparison}

\subsubsection{Data Preparation}
It is well known that the input of CNNs is patch-wise in HS image classification, and the output is the set of one-shot encoded labels. Unlike CNNs, GCNs feed pixel-wise samples into the network with an adjacency matrix that models the relations between samples and needs to be computed before the training process starts.

\subsubsection{Feature Representation} CNNs can extract rich spatial and spectral information from HS images in a short-range region, while GCNs are capable of modeling middle-range and long-range spatial relations between samples by means of their graph structure. Fig. \ref{fig:spatial_relation} illustrates such short-range, middle-range, and long-range relations in a HS scene.

\subsubsection{Network Training}
\label{a3}

CNNs, as the main member of the DL family, are normally trained through the use of mini-batch strategies. Conversely, GCNs only allow for full-batch network training, since all samples need to be simultaneously fed into the network.  

\subsubsection{Computational Cost}
\label{a4}
The computational cost of CNNs and GCNs in one layer is mainly dominated by matrix products, yielding an overall $\mathcal{O}(NDP)$ and
$\mathcal{O}(NDP+N^{2}D)$, respectively. $N$, $D$ and $P$ denote the sample number, and the dimensions of the input and output features, respectively. Evidently, GCNs are computationally complex for large graphs as compared to CNNs due to the large-sized matrix multiplication. To this end, a feasible solution might be the mini-batch strategy performed in GCNs. If possible, the complexity of GCNs can be greatly reduced to $\mathcal{O}(NDP+NMD)$, where $M\ll N$ denotes the size of mini-batches, thus having approximately same order as CNNs with respect to $N$.

\begin{figure}[!t]
	  \centering
		\subfigure{
			\includegraphics[width=0.4\textwidth]{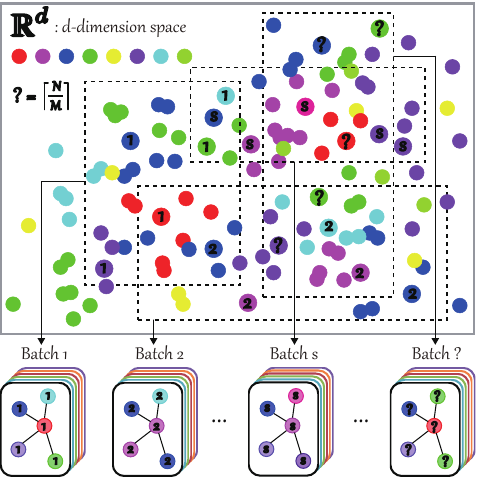}
		}
        \caption{An example illustrating how miniGCNs sample the subgraphs (or batches) from a full graph $\mathcal{G}$, aiming at training networks in a mini-batch fashion. The solid circles in different colors denote spectral signatures of different classes in high-dimensional feature space, while the dashed boxes represent random sampling regions for each batch.}
\label{fig:miniGCNs}
\end{figure}

\begin{figure*}[!t]
	  \centering
		\subfigure{
			\includegraphics[width=0.98\textwidth]{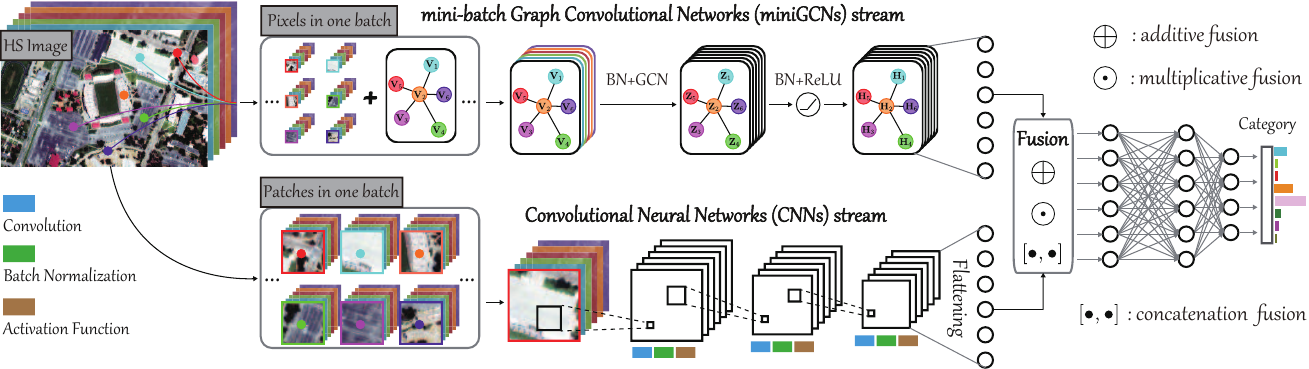}
		}
        \caption{An overview of our end-to-end fusion network (FuNet), illustrating one batch training iteration. It comprises feature extraction and fusion modules, where the former can extract different kinds of features (using both CNNs and miniGCNs) and the latter combines the resulting features using different fusion strategies before the final classification.}
\label{fig:fusionNet}
\end{figure*}

\subsection{Proposed MiniGCNs}
According to subsections \ref{a3} and \ref{a4}, the computational cost of GCNs becomes high with an increase in the size of the graphs. To circumvent the computational burden on large graphs, a feasible solution (in analogy to CNNs) is to use mini-batch processing. Inspired by inductive learning \cite{michalski1983theory}, we propose miniGCNs, making GCNs trainable in a mini-batch fashion. Note that our inductive setting neither exploits features nor graph information of testing nodes in the training process.

Before presenting the new update rule of graph convolution in the proposed miniGCNs, we first cast a proposition -- proved in \cite{zeng2019graphsaint} -- to theoretically guarantee the applicability of the mini-batch training strategy used in our miniGCNs. Given a full graph $\mathcal{G}$ with $|\mathcal{V}|=N$ on the labeled set, we construct a random node sampler with a budget $M$ ($M\ll N$). Before training each epoch, we repeatedly apply the sampler to $\mathcal{G}$ until each node is sampled, yielding a set of subgraphs $\mathbb{G}=\{\mathcal{G}_s=(\mathcal{V}_s,\mathcal{E}_s)|s=1,\ldots,\ceil*{\frac{N}{M}}\}$, where $\ceil*\bullet$ denotes the ceiling operation.

\begin{proposition}
Given a node $v$ sampled from a certain subgraph $\mathcal{V}_{s}$, i.e., $v\in\mathcal{V}_s$, an unbiased estimator of the node $v$ in the full-batch $(\ell+1)^{th}$ GCN layer, denoted as $\mathbf{z}_{v}^{(\ell+1)}$, can be computed by aggregating features between $v$ and all nodes $u\in\mathcal{V}_s$ in the $\ell^{th}$ layer:
\begin{equation}
\label{eq18}
\begin{aligned}
    \mathbf{z}^{(\ell+1)}_v=\sum_{u\in\mathcal{V}_s}\frac{(\mathbf{\widetilde{D}}^{-\frac{1}{2}}\mathbf{\widetilde{A}}\mathbf{\widetilde{D}}^{-\frac{1}{2}})_{uv}}{e_{uv}}\mathbf{z}^{(\ell)}_u\mathbf{W}^{(\ell)}+\mathbf{b}_{u}^{(\ell)},
\end{aligned}
\end{equation}
i.e., 
$\mathbb{E}(\mathbf{z}^{(\ell+1)}_v)=\sum_{u\in\mathcal{V}}(\mathbf{\widetilde{D}}^{-\frac{1}{2}}\mathbf{\widetilde{A}}\mathbf{\widetilde{D}}^{-\frac{1}{2}})_{uv}\mathbf{z}^{(\ell)}_u\mathbf{W}^{(\ell)}+\mathbf{b}^{(\ell)}$, if the constant of normalization $e_{uv}$ is set to $C_{uv}/{C_v}$, where $C_{uv}$ and $C_v$ are defined as the number of times that node or edge occurs in all sampled subgraphs.
\end{proposition}
% \begin{proposition}
%  If $e_{uv}=C_{uv}/{C_v}$, then $\mathbf{z}^{(l+1)}_v$ is an unbiased estimator of the aggregation of $v$ in the full $(l+1)^{th}$ GCN layer, i.e., 
%  $\mathbb{E}(\mathbf{z}^{(l+1)}_v)=\sum_{u\in\mathcal{V}}(\mathbf{\widetilde{D}}^{-1}\mathbf{\widetilde{A}})_{uv}\mathbf{\tilde{z}}^{(l)}_u+\mathbf{b}^{(l)}$.
% \begin{equation}
% \label{eq18}
% \begin{aligned}
%       \mathbf{z}^{(l+1)}_v=\sum_{u\in\mathcal{V}_s}\frac{(\mathbf{\widetilde{D}}^{-\frac{1}{2}}\mathbf{\widetilde{A}}\mathbf{\widetilde{D}}^{-\frac{1}{2}})_{uv}}{e_{uv}}\mathbf{z}^{(l)}_u\mathbf{W}^{(l)}+\mathbf{b}_{u}^{(l)},
% \end{aligned}
% \end{equation}
% where $\sum$ is the feature aggregation operator, $e_{uv}>0$ and $\mathcal{V}_s$ denote the aggregator normalization and the sampled nodes in subgraph $\mathcal{G}_s$, respectively. 
% %A normalized minibatch loss is further introduced as $L_{batch}=\sum_{v\in\mathcal{V}_s}L_v/\lambda_v$, where $\lambda_v=C_v|\mathcal{V}|/|\mathbb{G}|$. This can ensure an an unbiased estimator of the loss in full graph GCN effectively.
% \end{proposition}
With \textbf{Proposition 1} in mind, our miniGCNs can perform graph convolution in batches, just like CNNs. Using Eq. (\ref{eq17}), the update rule in one batch can be directly given by:
\begin{equation}
\label{eq19}
\begin{aligned}
     \mathbf{\widetilde{H}}_{s}^{(\ell+1)}=
        h(\mathbf{\widetilde{D}}_{s}^{-\frac{1}{2}}\mathbf{\widetilde{A}}_{s}\mathbf{\widetilde{D}}_{s}^{-\frac{1}{2}}\mathbf{\widetilde{H}}_{s}^{(\ell)}\mathbf{W}^{(\ell)}+\mathbf{b}_{s}^{(\ell)}),
\end{aligned}
\end{equation}
where $s$ is not only the $s^{th}$ subgraph, but also the $s^{th}$ batch in the network training. Note that we consider a special case of \textbf{Proposition 1}: random node sampling without replacement, by simply setting $C_{uv}=C_v=1$, i.e., $e_{uv}=1$. 
%hence the normalization constant in mini-batch loss is equal to $N/\ceil*{\frac{N}{M}}$.

By collecting the outputs of all batches, the final output in the $(\ell+1)^{th}$ layer can be reformulated as:
\begin{equation}
\label{eq20}
\begin{aligned}
      \mathbf{H}^{(\ell+1)}=[\mathbf{\widetilde{H}}_{1}^{(\ell+1)}, \cdots,\mathbf{\widetilde{H}}_{s}^{(\ell+1)},\cdots, \mathbf{\widetilde{H}}_{\ceil*{\frac{N}{M}}}^{(\ell+1)}].
\end{aligned}
\end{equation}
Fig. \ref{fig:miniGCNs} illustrates the process of batch generation in the proposed miniGCNs. This batch process is similar to the one adopted in CNNs, and the main difference lies in the fact that the graph or adjacency matrix in the obtained batch needs to be reassembled according to the connectivity of $\mathcal{G}$ after each sampling. 

%It should be noted, however, that to fully utilize the global graph information and stabilize the network training, we also consider all training samples as one batch if and only if the size of training set is not too large.

\subsection{MiniGCNs meet CNNs: End-to-end Fusion Networks}
Different network architectures are capable of extracting distinctive representations of HS images, e.g., spatial-spectral features in CNNs or topological relations between samples in GCNs. Generally speaking, a single model may not provide optimal results in terms of performance due to the lack of feature diversity. 

In this subsection, we naturally propose to fuse different models or features to enhance feature discrimination ability by jointly training CNNs and GCNs. Unlike traditional GCNs, the proposed miniGCNs can perform mini-batch learning and be combined with standard CNN models. This yields an end-to-end fusion network, called FuNet hereinafter. Three fusion strategies: additive (A), element-wise multiplicative (M), and concatenation (C) are considered. The three fusion models (A, M, and C) can be respectively formulated as follows:
\begin{equation}
\label{eq21}
\begin{aligned}
      \mathbf{H}_{FuNet-A}^{(\ell+1)}=\mathbf{H}_{CNNs}^{(\ell)} \oplus \mathbf{H}_{miniGCNs}^{(\ell)},
\end{aligned}
\end{equation}
\begin{equation}
\label{eq22}
\begin{aligned}
      \mathbf{H}_{FuNet-M}^{(\ell+1)}=\mathbf{H}_{CNNs}^{(\ell)} \odot \mathbf{H}_{miniGCNs}^{(\ell)},
\end{aligned}
\end{equation}
\begin{equation}
\label{eq23}
\begin{aligned}
      \mathbf{H}_{FuNet-C}^{(\ell+1)}=[\mathbf{H}_{CNNs}^{(\ell)}, \mathbf{H}_{miniGCNs}^{(\ell)}],
\end{aligned}
\end{equation}
where the operators $\oplus$, $\odot$, and $[\cdot,\cdot]$ respectively denote the element-wise addition, element-wise multiplication, and concatenation. Accordingly, $\mathbf{H}_{CNNs}^{(\ell)}$ and $\mathbf{H}_{miniGCNs}^{(\ell)}$ are represented as the $\ell^{th}$ layer features extracted from CNNs and miniGCNs, respectively.

Fig. \ref{fig:fusionNet} illustrates one batch training iteration of CNNs and miniGCNs in our newly proposed end-to-end fusion networks. As it can be seen, it comprises feature extraction and fusion modules, where the former can extract different kinds of features (using both CNNs and miniGCNs) and the latter combines the resulting features using different fusion strategies before the final classification.

\begin{table}[!t]
\centering
\caption{Land-cover classes of the Indian Pines dataset, with the number of training and test samples shown for each class.}
\begin{tabular}{c||ccc}
\toprule[1.5pt]
Class No.&Class Name&Training&Testing\\
\hline \hline 1&Corn Notill&50&1384\\
 2&Corn Mintill&50&784\\
 3&Corn&50&184\\
 4&Grass Pasture&50&447\\
 5&Grass Trees&50&697\\
 6&Hay Windrowed&50&439\\
 7&Soybean Notill&50&918\\
 8&Soybean Mintill&50&2418\\
 9&Soybean Clean&50&564\\
 10&Wheat&50&162\\
 11&Woods&50&1244\\
 12&Buildings Grass Trees Drives&50&330\\
 13&Stone Steel Towers&50&45\\
 14&Alfalfa&15&39\\
 15&Grass Pasture Mowed&15&11\\
 16&Oats&15&5\\
\hline \hline &Total&695&9671\\
\bottomrule[1.5pt]
\end{tabular}
\label{Table:Indian}
\end{table}

\begin{table}[!t]
\centering
\caption{Land-cover classes of the Pavia University dataset, with the number of training and test samples shown for each class.}
\begin{tabular}{c||ccc}
\toprule[1.5pt]
Class No.&Class Name&Training&Testing\\
\hline \hline 1&Asphalt&548&6304\\
 2&Meadows&540&18146\\
 3&Gravel&392&1815\\
 4&Trees&524&2912\\
 5&Metal Sheets&265&1113\\
 6&Bare Soil&532&4572\\
 7&Bitumen&375&981\\
 8&Bricks&514&3364\\
 9&Shadows&231&795\\
\hline \hline &Total&3921&40002\\
\bottomrule[1.5pt]
\end{tabular}
\label{Table:Pavia}
\end{table}

\begin{table}[!t]
\centering
\caption{Land-cover classes of the Houston2013 dataset, with the number of training and test samples shown for each class.}
\begin{tabular}{c||ccc}
\toprule[1.5pt]
Class No.&Class Name&Training&Testing\\
\hline \hline 1&Healthy Grass&198&1053\\
 2&Stressed Grass&190&1064\\
 3&Synthetic Grass&192&505\\
 4&Tree&188&1056\\
 5&Soil&186&1056\\
 6&Water&182&143\\
 7&Residential&196&1072\\
 8&Commercial&191&1053\\
 9&Road&193&1059\\
 10&Highway&191&1036\\
 11&Railway&181&1054\\
 12&Parking Lot1&192&1041\\
 13&Parking Lot2&184&285\\
 14&Tennis Court&181&247\\
 15&Running Track&187&473\\
\hline \hline &Total&2832&12197\\
\bottomrule[1.5pt]
\end{tabular}
\label{Table:H2013}
\end{table}

\section{Experiments}

\subsection{Data Description}
Three widely used HS datasets are adopted to assess the classification performance of our proposed algorithms, both quantitatively and qualitatively. 

\subsubsection{Indian Pines Dataset} The first HS dataset was acquired by the Airborne Visible/Infrared Imaging Spectrometer (AVIRIS) sensor over northwestern Indiana, USA. The scene comprises of $145\times 145$ pixels with a ground sampling distance (GSD) of 20 $m$ and 220 spectral bands in the wavelength range from 400 $nm$ to 2500 $nm$, at 10 $nm$ spectral resolution. We retain 200 channels by removing 20 noisy and water absorption bands, i.e., 104-108, 150-163, and 220. Table \ref{Table:Indian} lists 16 main land-cover categories involved in this studied scene, as well as the number of training and testing samples used for the classification task. Correspondingly, Fig. \ref{fig:data} shows a false-color image of this scene and the spatial distribution of training and test samples.

\subsubsection{Pavia University Dataset} The second HS scene is the well-known Pavia University, which was acquired by the Reflective Optics System Imaging Spectrometer (ROSIS) sensor. The ROSIS sensor acquired 103 bands covering the spectral range from $430nm$ to $860nm$, and the scene consists of $610\times 340$ pixels at GSD of 1.3 $m$. Moreover, there are 9 land cover classes in the scene. The class name and the number of training and test sets are detailed in Table \ref{Table:Pavia}, while the distribution of these samples is shown in Fig. \ref{fig:data}.

\subsubsection{Houston2013 Dataset} This dataset was used for the 2013 IEEE GRSS data fusion contest\footnote{http://www.grss-ieee.org/community/technical-committees/data-fusion/2013-ieee-grss-data-fusion-contest/}, and was collected using the ITRES CASI-1500 sensor over the campus of University of Houston and its surrounding rural areas in Texas, USA. The image size is $349\times 1905$ pixels with $144$ spectral bands ranging from 364 $nm$ to 1046 $nm$, at 10 $nm$ spectral resolution. It should be noted that the used dataset is a cloud-free HS product, processed by removing some small structures according to the illumination-related threshold maps computed based on the spectral signatures\footnote{The data were provided by Prof. N. Yokoya from the University of Tokyo.}. Table \ref{Table:H2013} lists 15 challenging land-cover categories and the training and test sets. In Fig. \ref{fig:data}, we show a false-color image of the HS scene and the corresponding distribution of the training and test samples.

\begin{figure}[!t]
	  \centering
		\subfigure{
			\includegraphics[width=0.48\textwidth]{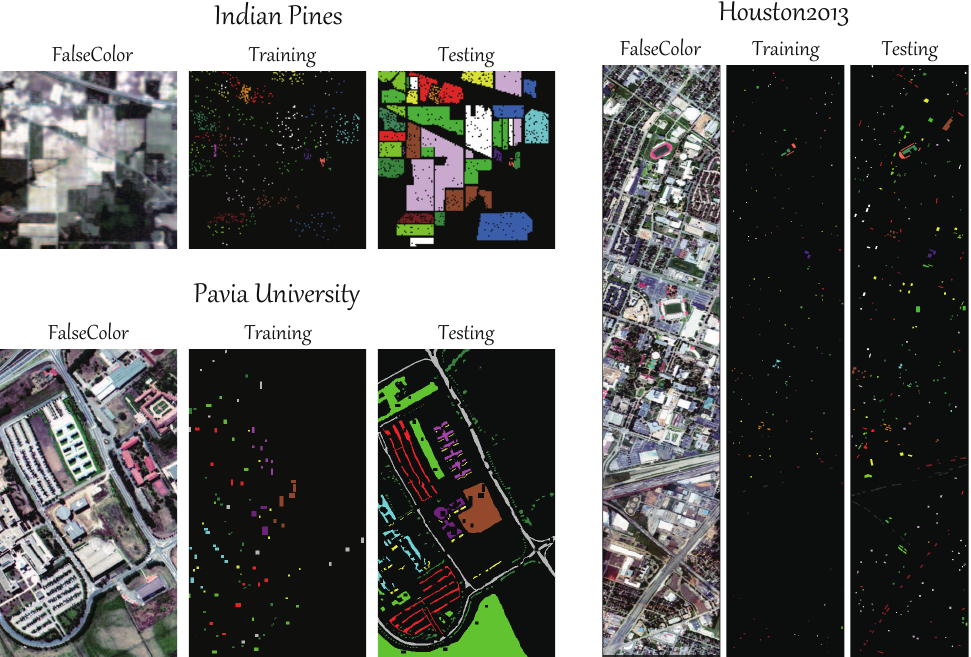}
		}
        \caption{False-color images and the distribution of training and test sets on the three considered datasets, i.e., Indian Pines, Pavia University, and Houston2013.}
\label{fig:data}
\end{figure}

\subsection{Experimental Setup}

\subsubsection{Implementation details}
All networks considered in this paper are implemented usig the Tensorflow platform, and Adam \cite{kingma2014adam} is used to optimize the networks.  By following the ``exponential'' learning rate policy, the current learning rate can be dynamically updated by multiplying a base learning rate (e.g., 0.001) by $(1-\frac{iter}{maxIter})^{0.5}$ at intervals of 50 epochs. In the process of network training, the maximum number of epochs is set to 200. Batch normalization (BN) \cite{ioffe2015batch} is adopted with the 0.9 momentum, and the batch size in the training phase is set to 32. Moreover, the $\ell_2$-norm regularization, set to 0.001, is employed on weights to stabilize the network training and reduce overfitting.

Note that the size for each layer and the hyperparameters in networks, such as learning rate and regularization, can be determined by 10-fold cross-validation, e.g., using a grid search on the validation set. Ten replications are performed to randomly separate the original training set into the new training set and validation set, with a percentage of 80\%-20\%. More specifically, we perform cross-validation to select the size of each layer and hyperparameters in the range of $\{16, 32, 64, 128, 256\}$ and $\{0.0001, 0.001, 0.01, 0.1, 1\}$, respectively. More details regarding the parameter settings can refer to our toolbox (or codes) that will be released after publication.

Furthermore, three commonly-used indices, i.e., \textit{Overall Accuracy (OA)}, \textit{Average Accuracy (AA)}, \textit{Kappa Coefficient ($\kappa$)}, are used to evaluate the classification performance quantitatively.

\begin{table}[!t]
\centering
\caption{General network configuration in each layer of our FuNet. FC, Conv, and MaxPool stand for fully connected, convolution, and max pooling, respectively, while $D$ and $P$ denote the input and output dimension in the networks, respectively. Furthermore, the last component in each block represents the output size.}
\resizebox{0.48\textwidth}{!}{ 
\begin{tabular}{c||c|c|c}
\toprule[1.5pt]
\multicolumn{2}{c|}{End-to-end Fusion Networks (FuNet)} & CNNs & miniGCNs\\
\hline 
\multicolumn{2}{c|}{Input Dimension} & $7\times 7\times D$ & $D$\\
\hline 
\multirow{15}{*}{Feature Extraction} & \multirow{5}{*}{Block1} & $3\times 3$ Conv & BN\\
& & BN & Graph Conv\\
& & $2\times 2$ MaxPool & BN\\
& & ReLU & ReLU\\
& & $4\times 4\times 32$ & $128$\\
\cline{2-4}
& \multirow{5}{*}{Block2} & $3\times 3$ Conv & --\\
& & BN & --\\
& & $2\times 2$ MaxPool & --\\
& & ReLU & --\\
& & $2\times 2\times 64$ & --\\
\cline{2-4}
& \multirow{5}{*}{Block3} & $1\times 1$ Conv & --\\
& & BN & --\\
& & $2\times 2$ MaxPool & --\\
& & ReLU & --\\
& & $1\times 1\times 128$ & --\\
\hline \hline
\multirow{7}{*}{Feature Fusion} & \multirow{4}{*}{Block4} & \multicolumn{2}{c}{FC Encoder} \\
& & \multicolumn{2}{c}{BN}\\
& & \multicolumn{2}{c}{ReLU}\\
& & \multicolumn{2}{c}{$128$}\\
\cline{2-4}
& \multirow{3}{*}{Block5} & \multicolumn{2}{c}{FC Encoder} \\
& & \multicolumn{2}{c}{Softmax}\\
& & \multicolumn{2}{c}{$P$}\\
\hline
\multicolumn{2}{c|}{Ouput Dimension} & \multicolumn{2}{c}{$P$}\\
\bottomrule[1.5pt]
\end{tabular}
}
\label{tab:network_configuration}
\end{table}

\begin{table*}[!t]
\centering
\caption{Quantitative comparison of different algorithms in terms of OA, AA, and $\kappa$ on the Indian Pines dataset. The best one is shown in bold.}
\begin{tabular}{c||ccccccc|cccc}
\toprule[1.5pt] Class No. & KNN & RF & SVM & 1-D CNN & 2-D CNN & 3-D CNN & GCN & miniGCN & FuNet-A & FuNet-M & FuNet-C\\
\hline \hline
1 & 45.45 & 57.80 & 67.34 & 47.83 & 65.90 & 66.26 & 65.97 & \bf 72.54 & 68.64 & 69.51 & 68.50\\
2 & 46.94 & 56.51 & 67.86 & 42.35 & 76.66 & 71.94 & 72.70 & 55.99 & 80.99 & \bf 82.40 & 79.59\\
3 & 77.72 & 80.98 & 93.48 & 60.87 & 92.39 & 97.28 & 87.50 & 92.93 & 95.11 & 94.57 & \bf 99.46\\
4 & 84.56 & 85.68 & 94.63 & 89.49 & 93.96 & 95.08 & 93.74 & 92.62 & \bf 96.64 & 96.42 & 95.08\\
5 & 80.06 & 79.34 & 88.52 & 92.40 & 87.23 & 88.09 & 91.39 & 94.98 & 95.41 & \bf 96.99 & 95.70\\
6 & 97.49 & 95.44 & 94.76 & 97.04 & 97.27 & 98.18 & 97.49 & 98.63 & 99.32 & \bf 99.54 & \bf 99.54\\
7 & 64.81 & \bf 77.56 & 73.86 & 59.69 & 77.23 & 75.38 & 75.38 & 64.71 & 72.98 & 76.80 & 75.93\\
8 & 48.68 & 58.85 & 52.07 & 65.38 & 57.03 & 56.29 & 51.70 & 68.78 & \bf 70.31 & 58.97 & 68.90\\
9 & 44.33 & 62.23 & 72.70 & \bf 93.44 & 72.87 & 78.01 & 62.77 & 69.33 & 74.82 & 74.82 & 71.63\\
10 & 96.30 & 95.06 & 98.77 & 99.38 & \bf 100.00 & \bf 100.00 & 96.91 & 98.77 & 99.38 & 99.38 & 99.38\\
11 & 74.28 & 88.75 & 86.17 & 84.00 & \bf 92.85 & 90.59 & 86.25 & 87.78 & 85.93 & 79.50 & 89.55\\
12 & 15.45 & 54.24 & 71.82 & 86.06 & 88.18 & 90.30 & 66.97 & 50.00 & \bf 93.03 & 91.21 & 91.52\\
13 & 91.11 & 97.78 & 95.56 & 91.11 & \bf 100.00 & \bf 100.00 & 95.56 & \bf 100.00 & \bf 100.00 & \bf 100.00 & \bf 100.00\\
14 & 33.33 & 56.41 & 82.05 & 84.62 & 84.62 & 74.36 & 71.79 & 48.72 & 79.49 & 82.05 & \bf 94.87\\
15 & 81.82 & 81.82 & 90.91 & \bf 100.00 & \bf 100.00 & \bf 100.00 & 81.82 & 72.73 & \bf 100.00 & \bf 100.00 & \bf 100.00\\
16 & 40.00 & \bf 100.00 & \bf 100.00 & 80.00 & \bf 100.00 & \bf 100.00 & \bf 100.00 & 80.00 & \bf 100.00 & \bf 100.00 & \bf 100.00\\
\hline \hline
OA (\%) & 59.17 & 69.80 & 72.36 & 70.43 & 75.89 & 75.48 & 71.97 & 75.11 & 79.76 & 76.76 & \bf 79.89\\
AA (\%) & 63.90 & 76.78 & 83.16 & 79.60 & 86.64 & 86.36 & 81.12 & 78.03 & 88.25 & 87.64 & \bf 89.35\\
$\kappa$ & 0.5395 & 0.6591 & 0.6888 & 0.6642 & 0.7281 & 0.7240 & 0.6852 & 0.7164 & 0.7698 & 0.7382 & \bf 0.7716\\
\bottomrule[1.5pt]
\end{tabular}
\label{tab:IP}
\end{table*}

\begin{table*}[!t]
\centering
\caption{Quantitative performance comparison of different algorithms in terms of OA, AA, and $\kappa$ on the Pavia University dataset. The best one is shown in bold.}
\begin{tabular}{c||ccccccc|cccc}
\toprule[1.5pt] Class No. & KNN & RF & SVM & 1-D CNN & 2-D CNN & 3-D CNN & GCN & miniGCN & FuNet-A & FuNet-M & FuNet-C\\
\hline \hline
1 & 73.86 & 79.81 & 74.22 & 88.90 & 80.98 & 80.69 & 76.49 & 96.35 & \bf 96.99 & 96.47 & 96.67\\
2 & 64.31 & 54.90 & 52.79 & 58.81 & 81.70 & 89.12 & 70.15 & 89.43 & \bf 97.74 & 97.36 & 97.60\\
3 & 55.10 & 46.34 & 65.45 & 73.11 & 67.99 & 65.90 & 62.70 & \bf 87.01 & 83.98 & 83.44 & 84.49\\
4 & 94.95 & \bf 98.73 & 97.42 & 82.07 & 97.36 & 98.45 & 98.35 & 94.26 & 96.45 & 84.40 & 89.95\\
5 & 99.19 & 99.01 & 99.46 & 99.46 & 99.64 & 99.19 & 99.37 & 99.82 & 99.55 & \bf 100.00 & 99.64\\
6 & 65.16 & 75.94 & 93.48 & \bf 97.92 & 97.59 & 92.37 & 83.22 & 43.12 & 71.33 & 85.30 & 90.56\\
7 & 84.30 & 78.70 & 87.87 & 88.07 & 82.47 & 76.04 & 88.38 & \bf 90.96 & 66.67 & 63.80 & 78.27\\
8 & 84.10 & 90.22 & 89.39 & 88.14 & \bf 97.62 & 95.81 & 92.33 & 77.42 & 69.61 & 71.53 & 71.73\\
9 & 98.36 & 97.99 & \bf 99.87 & \bf 99.87 & 95.60 & 95.72 & 95.72 & 87.27 & 99.86 & 99.22 & 98.04\\
\hline \hline
OA (\%) & 70.53 & 69.67 & 70.82 & 75.50 & 86.05 & 88.44 & 77.99 & 79.79 & 89.00 & 90.34 & \bf 92.20\\
AA (\%) & 79.68 & 80.18 & 84.44 & 86.26 & 88.99 & 88.14 & 85.19 & 85.07 & 86.91 & 86.84 & \bf 89.66\\
$\kappa$ & 0.6268 & 0.6237 & 0.6423 & 0.6948 & 0.8187 & 0.8472 & 0.7196 & 0.7367 & 0.8540 & 0.8709 & \bf 0.8951\\
\bottomrule[1.5pt]
\end{tabular}
\label{tab:PU}
\end{table*}

\begin{table*}[!t]
\centering
\caption{Quantitative performance comparison of different algorithms in terms of OA, AA, and $\kappa$ on the Houston2013 dataset. The best one is shown in bold.}
\begin{tabular}{c||ccccccc|cccc}
\toprule[1.5pt] Class No. & KNN & RF & SVM & 1-D CNN & 2-D CNN & 3-D CNN & GCN & miniGCN & FuNet-A & FuNet-M & FuNet-C\\
\hline \hline
1 & 83.19 & 83.38 & 83.00 & 87.27 & 85.09 & 84.71 & 90.14 & \bf 98.39 & 84.33 & 83.86 & 85.75\\
2 & 95.68 & 98.40 & 98.40 & 98.21 & 99.91 & 99.34 & 99.08 & 92.11 & \bf 100.00 & 98.59 & 99.44\\
3 & 99.41 & 98.02 & 99.60 & \bf 100.00 & 77.23 & 84.55 & 79.94 & 99.60 & 82.57 & 83.37 & 80.79\\
4 & 97.92 & 97.54 & 98.48 & 92.99 & 97.73 & 98.01 & 96.69 & 96.78 & 98.48 & \bf 98.96 & 98.58\\
5 & 96.12 & 96.40 & 97.82 & 97.35 & 99.53 & 97.82 & 86.18 & 97.73 & 98.86 & \bf 99.72 & 99.24\\
6 & 92.31 & \bf 97.20 & 90.91 & 95.10 & 92.31 & 93.01 & 33.33 & 95.10 & 95.80 & 96.50 & 95.10\\
7 & 80.88 & 82.09 & 90.39 & 77.33 & 92.16 & 86.29 & \bf 97.09 & 57.28 & 88.43 & 89.55 & 91.60\\
8 & 48.62 & 40.65 & 40.46 & 51.38 & 79.39 & 76.26 & 71.63 & 68.09 & 85.94 & \bf 89.36 & 74.83\\
9 & 72.05 & 69.78 & 41.93 & 27.95 & \bf 86.31 & 84.23 & 70.93 & 53.92 & 85.08 & 83.29 & 85.27\\
10 & 53.19 & 57.63 & 62.64 & \bf 90.83 & 43.73 & 74.32 & 72.17 & 77.41 & 72.30 & 79.25 & 79.25\\
11 & 86.24 & 76.09 & 75.43 & 79.32 & \bf 87.00 & 82.35 & 85.22 & 84.91 & 81.69 & 79.89 & 82.35\\
12 & 44.48 & 49.38 & 60.04 & 76.56 & 66.28 & 77.71 & 63.41 & 77.23 & 79.06 & \bf 79.15 & 78.87\\
13 & 28.42 & 61.40 & 49.47 & 69.47 & \bf 90.18 & 81.05 & 62.34 & 50.88 & \bf 90.18 & 87.72 & 89.12\\
14 & 97.57 & \bf 99.60 & 98.79 & 99.19 & 90.69 & 88.66 & 89.73 & 98.38 & 90.69 & 93.93 & 88.26\\
15 & 98.10 & 97.67 & 97.46 & 98.10 & 77.80 & 84.57 & \bf 99.36 & 98.52 & 93.66 & 98.94 & 86.68\\
\hline \hline
OA (\%) & 77.30 & 77.48 & 76.91 & 80.04 & 83.72 & 86.04 & 81.19 & 81.71 & 87.73 & \bf 88.62 & 87.39\\
AA (\%) & 78.28 & 80.35 & 78.99 & 82.74 & 84.35 & 86.19 & 79.82 & 83.09 & 88.47 & \bf 89.47 & 87.68\\
$\kappa$ & 0.7538 & 0.7564 & 74.94 & 0.7835 & 0.8231 & 0.8483 & 0.7962 & 0.8018 & 0.8668 & \bf 0.8764 & 0.8631\\
\bottomrule[1.5pt]
\end{tabular}
\label{tab:H2013}
\end{table*}

\begin{figure}[!t]
	  \centering
		\subfigure[GCNs]{
			\includegraphics[width=0.225\textwidth]{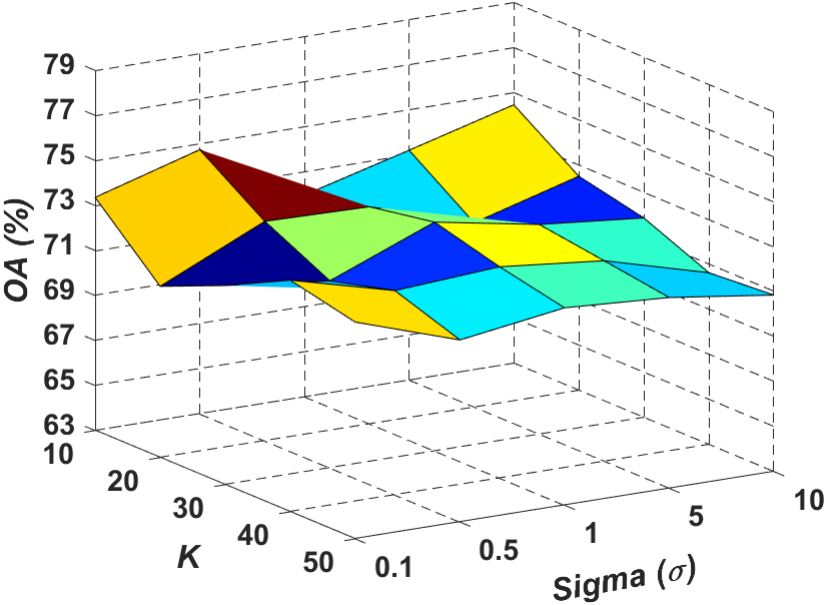}
		}
		\subfigure[miniGCNs]{
			\includegraphics[width=0.225\textwidth]{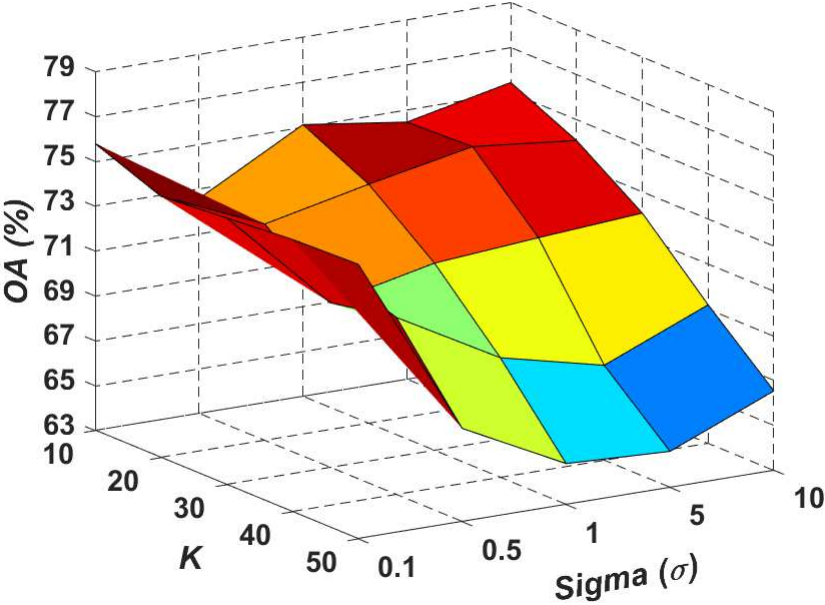}
		}
         \caption{Parameter sensitivity analysis (on the Indian Pines data) of the adjacency matrix $\mathbf{\widetilde{A}}$ [see Eq. (\ref{eq1})] in terms of $K$ and $\sigma$ for GCNs (a) and miniGCNs (b).}
\label{fig:para}
\end{figure}

\subsubsection{Comparison with state-of-the-art baseline methods}
Several state-of-the-art baseline methods have been selected for comparison, including K-nearest neighbor (KNN) classifier, random forest (RF), 1-D CNN, 2-D CNN, GCN, and our proposed miniGCN, as well as three different fusion network with different strategies: FuNet-A, FuNet-M, FuNet-C. The parameter settings are described below:
\begin{itemize}
    \item For the KNN, we set the number of nearest neighbors ($K$) to 10, to be consistent with that of $K$ in GCN-related methods, e.g., GCN, miniGCN, and FuNet.
    \item For the RF, 200 decision trees are used in the classifier.
    \item For the SVM, the well-known libsvm toolbox\footnote{https://www.csie.ntu.edu.tw/$\sim$cjlin/libsvm/} is used for implementation in our case. The considered SVM uses the RBF kernel, whose two optimal hyperparameters $\sigma$ and $\lambda$ (the regularization parameter to balance the training and testing errors) can be determined by five-fold cross validation in the range $\sigma=[2^{-3},2^{-2},\dots,2^{4}]$ and $\lambda=[10^{-2},10^{-1},\dots,10^{4}]$.
    \item For the 1-D CNN, we use one convolutional block including a 1-D convolutional layer with a filter size of 128, a BN layer, a ReLU activation layer, and a softmax layer with the size of $P$, where $P$ denotes the dimension of network output.
    \item For the 2-D CNN (similar to 1-D CNN), the architecture is composed of three 2-D convolutional blocks and a softmax layer. Each convolutional block involves a 2-D conventional layer, a BN layer, a max-pooling layer, and a ReLU activation layer. Moreover, the receptive fields along the spatial and spectral domains for each convolutional layer are $3\times 3 \times 32$, $3\times 3 \times 64$, and $1\times 1 \times 128$, respectively.
    \item For the 3-D CNN, we adopt the same network architecture as the one in \cite{chen2016deep}. The only difference lies in that we remove the dropout layer in each block to make a fair comparison with other networks, e.g., 2-D CNN.
    \item For the GCN, similar to \cite{kipf2016semi}, a graph convolutional hidden layer with 128 units is implemented in the GCN before feeding the features into the softmax layer, where the adjacency matrix $\mathbf{\widetilde{A}}$ can be computed by means of KNN-based graph ($K=10$ in our case). The graph convolution, GCN, and 1-D CNN share the same network configuration for a fair comparison.
    \item Our miniGCN has the same architecture as the GCN. The main difference between GCN and miniGCN lies in the fact that miniGCN is capable of training the networks in batch-wise fashion, and tends to reach a better local optimum of networks.
    \item To better exploit diverse information of HS images, e.g., features extracted from CNNs or GCNs, our FuNets with A, M, and C different fusion strategies are developed by additionally adding a fully-connected (FC) fusion layer behind CNNs and miniGCNs. Table \ref{tab:network_configuration} details the configuration of our FuNet for the layer-wise network architecture.
\end{itemize}

It should be noted, however, that the patch centered by a pixel is usually used as the input of CNNs in HS image classification. In this connection, the original HS image is extended by the ``replicate'' operation, i.e., copying the pixels within the image to that out of the original image boundary, to solve the problem of the boundaries in the CNNs-related experiments.

\begin{figure}[!t]
	  \centering
		\subfigure{
			\includegraphics[width=0.48\textwidth]{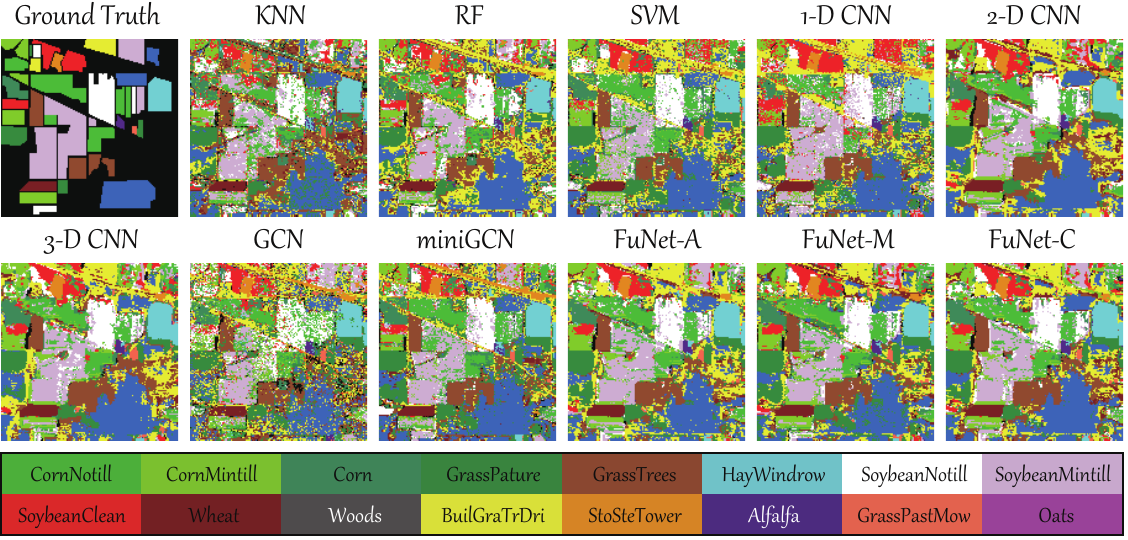}
		}
        \caption{Ground truth and classification maps obtained by different methods on the Indian Pines dataset.}
\label{fig:CM_IP}
\end{figure}

\begin{figure}[!t]
	  \centering
		\subfigure{
			\includegraphics[width=0.48\textwidth]{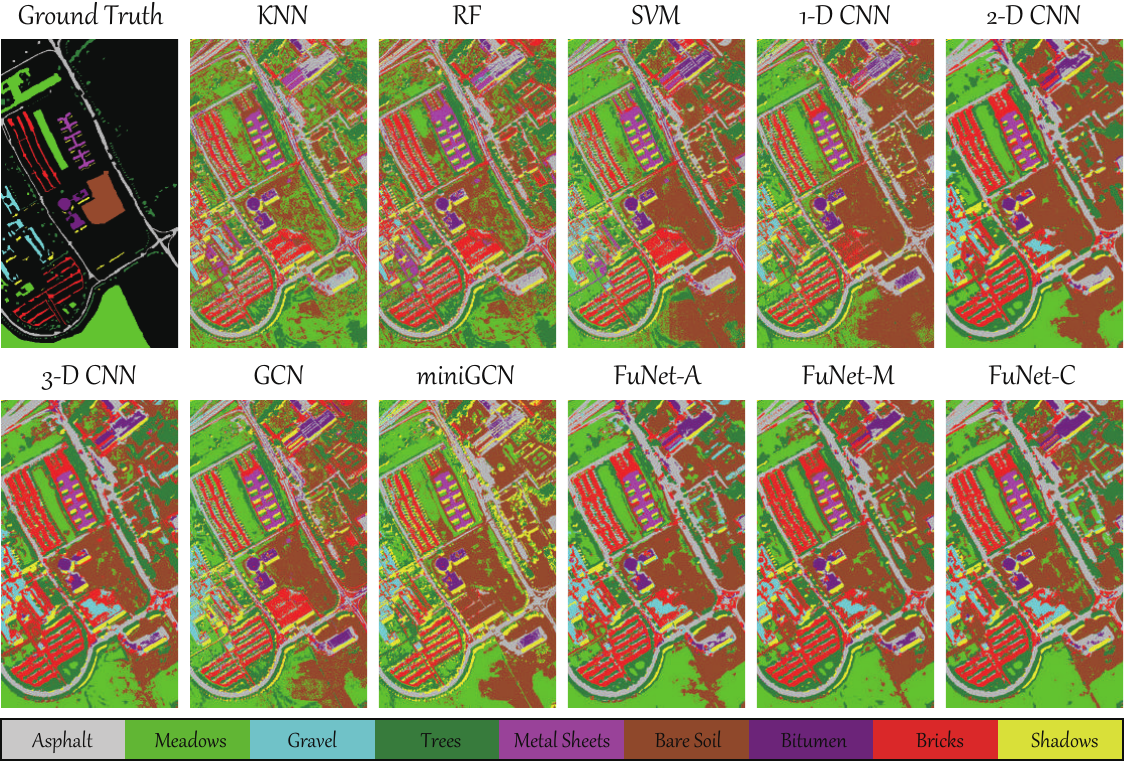}
		}
        \caption{Ground truth and classification maps obtained by different methods on the Pavia university dataset.}
\label{fig:CM_PU}
\end{figure}

\subsection{Parameter Analysis on $\mathbf{\widetilde{A}}$ Generation}
Since the performance of GCNs depends (to some extent) on the quality of adjacency matrix, i.e., $\mathbf{\widetilde{A}}$, it is important to investigate the performance gains that can be obtained by adjusting the two parameters: number of neighbors ($K$) and width of RBF function ($\sigma$) of $\mathbf{\widetilde{A}}$ [see Eq. (\ref{eq1})]. For this purpose, we show the changing trend (in terms of \textit{OA}) for different combinations of the two parameters in the Indian Pines data. More specifically, GCNs and miniGCNs are selected to analyze the parameter sensitivity. As it can be seen from Fig. \ref{fig:para}, the parameter $K$ (to a large extent) dominates the performance gain. Nevertheless, the OAs of GCNs and miniGCNs remain stable with an increase of $K$ value. On the other hand, varying the parameter $\sigma$ only yields a slight performance fluctuation, indicating that this parameter might not be correctly fine-tuned. Most importantly, we observed that the performance gain or derogation in miniGCNs is relatively slow and gentle with the gradual change of the two parameters. In turn, with different parameter combinations, the GCNs leads to comparatively more perturbed results. Moreover, the whole classification performance of GCNs also seems to reach a bottleneck, because its full-batch training strategy usually fails to find a better local optimum. Comprehensively, the parameter combination of $(K, \sigma)$ in our case is set to $(10, 1)$, since this parameter range is relatively stable and, hence, it is applied to the rest of considered datasets for simplicity. 

\begin{figure*}[!t]
	  \centering
		\subfigure{
			\includegraphics[width=0.98\textwidth]{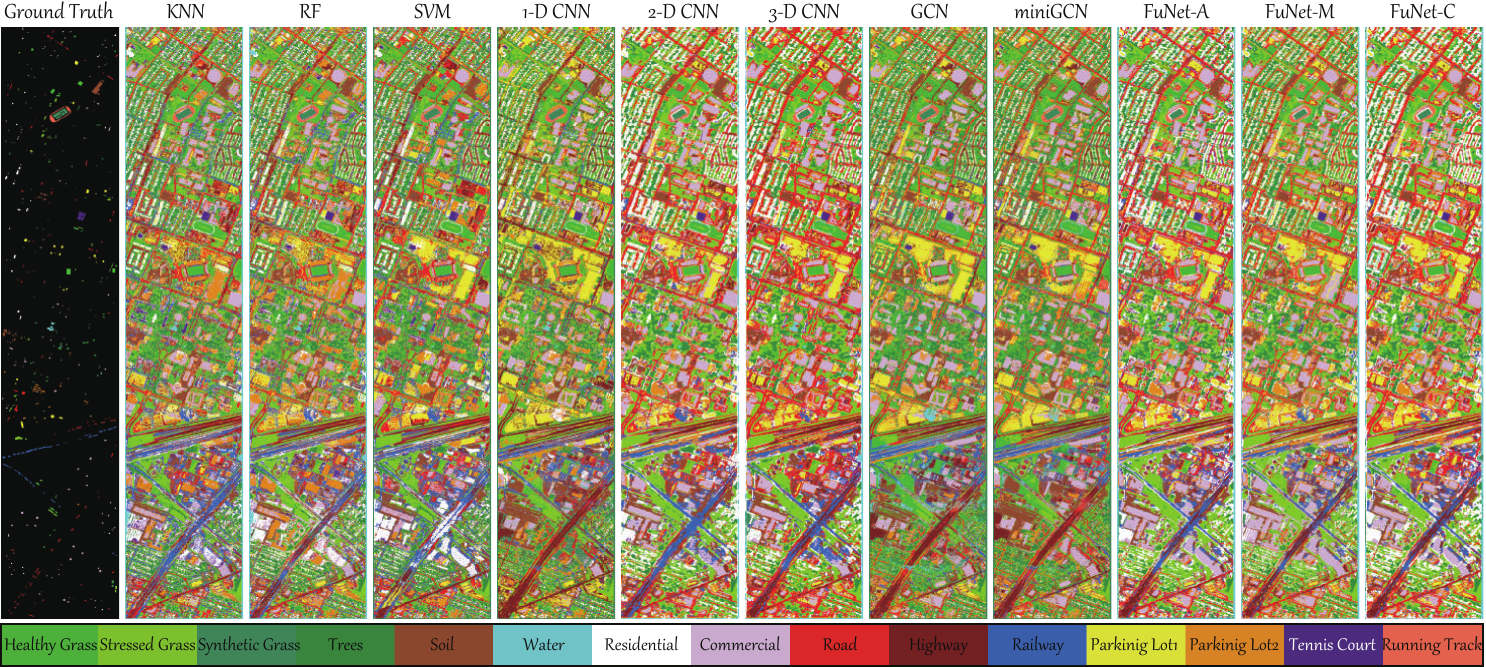}
		}
        \caption{Ground truth and classification maps obtained by different methods on the Houston2013 dataset.}
\label{fig:CM_HH}
\end{figure*}

\subsection{Quantitative Evaluation}
Tables \ref{tab:IP}, \ref{tab:PU}, and \ref{tab:H2013} quantitatively report the classification scores obtained by different methods in terms of \textit{OA}, \textit{AA}, and $\kappa$, as well as the individual class accuracies for the Indian Pines, Pavia University, and Houston2013 datasets, respectively.

Overall, KNN, RF, and SVM obtain similar classification results on the Pavia University and Houston2013 datasets, while the classification performance of the KNN classifier is inferior to that achieved using the RF on the Indian Pines dataset. This might be explained by a few noisy training samples. Please note that there is a similar trend between RF and SVM in classification performance. By means of the powerful learning ability of DL techniques, 1-D CNN, 2-D CNN, 3-D CNN, and GCN perform better than traditional classifiers (KNN, RF, and SVM). Unlike 1-D CNN and GCN, that only consider pixel-wise network input, 2-D CNN and 3-D CNN can extract the spatial-spectral information from HS images more effectively, yielding higher classification accuracies. Not surprisingly, the performance of 3-D CNN is generally superior to that of 2-D CNN, owing to the additional local convolution on the spectral domain. We have to point out, however, that the 3-D CNN requires additional network parameters to be estimated, and tends to suffer from overfitting problems (particularly with limited training samples). The resulting accuracies on the Indian Pines dataset demonstrate these potential problems. Moreover, GCN brings moderate increments of at least 1\% \textit{OA}, \textit{AA}, and $\kappa$ over the 1-D CNN, since the spatial relation between samples can be well-modeled in the form of a graph structure by GCNs. 

Remarkably, our miniGCN achieves stable performance improvements when compared to either GCN or 1-D CNN, even making it comparable to 2-D CNN to some extent, e.g., on the Indian Pines and Houston2013 datasets. As expected, the FuNet (that combines the benefits of CNNs and GCNs) outperforms those single models, demonstrating its ability to fuse different spectral representations. More specifically, a comparison between the three commonly-used fusion strategies reveals that FuNet-C tends to obtain better classification performance compared to FuNet-A and FuNet-M, particularly on the Indian Pines and Pavia University, where there is a dramatic performance improvement (\textit{cf.} Tables \ref{tab:IP} and \ref{tab:PU}).

Furthermore, for those classes that have very few samples, e.g., \textit{Alfalfa}, \textit{Grass Pasture Mowed}, \textit{Oats} on Indian Pines, or unbalanced samples e.g., \textit{Road}, \textit{Parking Lot2} on Houston2013, the 2-D CNN and 3-D CNN can obtain higher classification accuracies by considering the contextual information in both the spatial and spectral domains. On the contrary, the GCN-based models fail to accurately model those classes. But it is worth noting that the fused networks are capable of better identifying these challenging classes, due to the joint use of spatial-spectral (2-D CNN) and relation-augmented (miniGCN) features.

\subsection{Visual Comparison}

We also make a visual comparison between different classification methods in the form of classification maps, as shown in Figs. \ref{fig:CM_IP}, \ref{fig:CM_PU}, and \ref{fig:CM_HH}. In general, pixel-wise classification models (e.g., KNN, RF, SVM, 1-D CNN) result in salt and pepper noise in the classification maps. Although the GCN considers the spatial relation modeling between samples, the use of large graphs constructed based on all samples (and full-batch network training) limits its performance to a great extent, thereby yielding relatively poor classification maps. Our proposed miniGCN extracts the HS features by locally preserving the graph (or manifold) structure in one batch, leading to results that are comparable to those obtained by the 2-D CNN and 3-D CNN. This means that we can achieve relatively robust representations compared to full graph preservation, since the batch-wise strategy can eliminate some errors resulting from the manually-computed adjacency matrix, and further reduce the error accumulation and propagation between layers. As expected, the FuNet-based methods obtain smoother and more detailed maps in comparison with other competitors, mainly due to the effective combination of different features that further enhance the HS representation ability. It should be noted, however, that the batch-wise input in CNNs could lead to losing some edge details to some extent (e.g., 2-D CNN, and 3-D CNN). This explains why the classification maps obtained by FuNets are not as sharp (in terms of edge delination) as those obtained by only using miniGCNs.

\section{Conclusion}
Owing to the embedding of graph (or topological) structure, GCNs can properly characterize the underlying data structure of HS images in high-dimensional space, but inevitably introduce some drawbacks, e.g., high storage and computational cost when computing the adjacency matrix, gradient exploding or vanishing problems (due to full-batch network training), and the need to re-train these networks when new data are fed. In order to address these problems, in this paper we develop a new supervised version of GCNs, called miniGCNs, which allows us to train large-scale graph networks in a mini-batch fashion. Owing to their batch-wise network training strategy, our newly proposed miniGCNs are more flexible, in the sense that they not only yield lower computational cost and stable local optima in the training phase, but also can directly predict the new input samples, i.e., the out-of-sample cases, with no need to re-train the network. More significantly, our trainable mini-batch strategy makes it possible to jointly use CNNs and GCNs for extracting more diverse and discriminative feature representations for the HS image classification task. To exploit this property, we have further investigated several fusion modules: A, M, and C, that integrate CNNs and miniGCNs in an end-to-end trainable fashion. Our experimental results, conducted on three widely used HS datasets, demonstrate the effectiveness and superiority of our newly proposed miniGCNs as compared to the traditional GCNs. Also, the FuNet (with different fusion strategies) has been shown to be superior to using single model (e.g., CNNs, miniGCNs).

In the future, we will investigate the possible combination of different deep networks and our miniGCNs, and also develop more advanced fusion modules, e.g., weighted fusion, to fully exploit the rich spectral information contained in HS images.

\section*{Acknowledgments}

The authors would like to the Hyperspectral Image Analysis group at the University of Houston for providing the CASI University of Houston datasets and the IEEE GRSS DFC2013.

\bibliographystyle{ieeetr}
\bibliography{HDF_ref}

\begin{IEEEbiography}[{\includegraphics[width=1in,height=1.25in,clip,keepaspectratio]{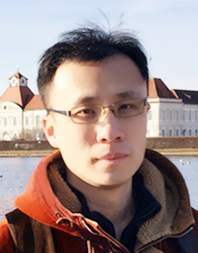}}]{Danfeng Hong}
(S'16-M'19) received the M.Sc. degree (summa cum laude) in computer vision, College of Information Engineering, Qingdao University, Qingdao, China, in 2015, the Dr. -Ing degree (summa cum laude) in Signal Processing in Earth Observation (SiPEO), Technical University of Munich (TUM), Munich, Germany, in 2019. 

Since 2015, he worked as a Research Associate at the Remote Sensing Technology Institute (IMF), German Aerospace Center (DLR), Oberpfaffenhofen, Germany. Currently, he is a research scientist and leads a Spectral Vision working group at IMF, DLR, and also an adjunct scientist in GIPSA-lab, Grenoble INP, CNRS, Univ. Grenoble Alpes, Grenoble, France. 

His research interests include signal / image processing and analysis, hyperspectral remote sensing, machine / deep learning, artificial intelligence and their applications in Earth Vision.
\end{IEEEbiography}

\vskip -2\baselineskip plus -1fil

\begin{IEEEbiography}[{\includegraphics[width=1in,height=1.25in,clip,keepaspectratio]{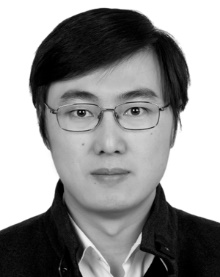}}]{Lianru Gao} (M'12-SM'18) received the B.S. degree in civil engineering from Tsinghua University, Beijing, China, in 2002, the Ph.D. degree in cartography and geographic information system from Institute of Remote Sensing Applications, Chinese Academy of Sciences (CAS), Beijing, China, in 2007.

He is currently a Professor with the Key Laboratory of Digital Earth Science, Aerospace Information Research Institute, CAS. He also has been a visiting scholar at the University of Extremadura, Cáceres, Spain, in 2014, and at the Mississippi State University (MSU), Starkville, USA, in 2016. His research focuses on hyperspectral image processing and information extraction. In last ten years, he was the PI of 10 scientific research projects at national and ministerial levels, including projects by the National Natural Science Foundation of China (2010-2012, 2016-2019, 2018-2020), and by the Key Research Program of the CAS (2013-2015). He has published more than 160 peer-reviewed papers, and there are more than 80 journal papers included by SCI. He was coauthor of an academic book entitled ``Hyperspectral Image Classification And Target Detection''. He obtained 28 National Invention Patents in China. He was awarded the Outstanding Science and Technology Achievement Prize of the CAS in 2016, and was supported by the China National Science Fund for Excellent Young Scholars in 2017, and won the Second Prize of The State Scientific and Technological Progress Award in 2018. He received the recognition of the Best Reviewers of the IEEE Journal of Selected Topics in Applied Earth Observations and Remote Sensing in 2015, and the Best Reviewers of the IEEE Transactions on Geoscience and Remote Sensing in 2017.
\end{IEEEbiography}

\vskip -2\baselineskip plus -1fil

\begin{IEEEbiography}[{\includegraphics[width=1in,height=1.25in,clip,keepaspectratio]{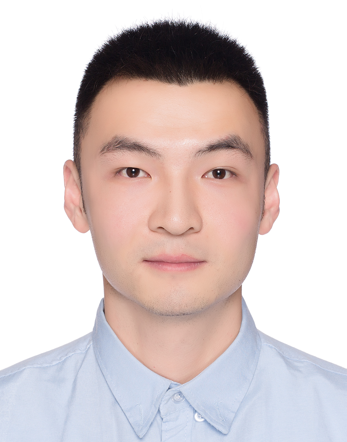}}]{Jing Yao} received the B.Sc. degree from Northwest University, Xi’an, China, in 2014. He is currently pursuing the Ph.D. degree with the School of Mathematics and Statistics, Xi’an Jiaotong University, Xi’an, China. 

From 2019 to 2020, he is a visiting student at Signal Processing in Earth Observation (SiPEO), Technical University of Munich (TUM), Munich, Germany, and at the Remote Sensing Technology Institute (IMF), German Aerospace Center (DLR), Oberpfaffenhofen, Germany.

His research interests include low-rank modeling, hyperspectral image analysis and deep learning-based image processing methods.
\end{IEEEbiography}

\vskip -2\baselineskip plus -1fil

\begin{IEEEbiography}[{\includegraphics[width=1in,height=1.25in,clip,keepaspectratio]{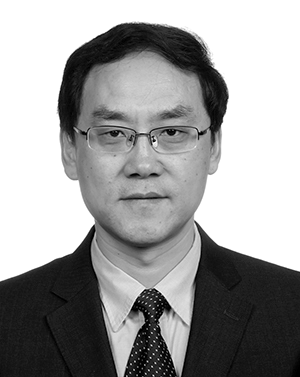}}]{Bing Zhang} (M'11–SM'12-F'19) received the B.S. degree in geography from Peking University, Beijing, China, in 1991, and the M.S. and Ph.D. degrees in remote sensing from the Institute of Remote Sensing Applications, Chinese Academy of Sciences (CAS), Beijing, China, in 1994 and 2003, respectively.

Currently, he is a Full Professor and the Deputy Director of the Aerospace Information Research Institute, CAS, where he has been leading lots of key scientific projects in the area of hyperspectral remote sensing for more than 25 years. His research interests include the development of Mathematical and Physical models and image processing software for the analysis of hyperspectral remote sensing data in many different areas. He has developed 5 software systems in the image processing and applications. His creative achievements were rewarded 10 important prizes from Chinese government, and special government allowances of the Chinese State Council. He was awarded the National Science Foundation for Distinguished Young Scholars of China in 2013, and was awarded the 2016 Outstanding Science and Technology Achievement Prize of the Chinese Academy of Sciences, the highest level of Awards for the CAS scholars.

Dr. Zhang has authored more than 300 publications, including more than 170 journal papers. He has edited 6 books/contributed book chapters on hyperspectral image processing and subsequent applications. He is the IEEE fellow and currently serving as the Associate Editor for IEEE Journal of Selected Topics in Applied Earth Observations and Remote Sensing. He has been serving as Technical Committee Member of IEEE Workshop on Hyperspectral Image and Signal Processing since 2011, and as the president of hyperspectral remote sensing committee of China National Committee of International Society for Digital Earth since 2012, and as the Standing Director of Chinese Society of Space Research (CSSR) since 2016. He is the Student Paper Competition Committee member in IGARSS from 2015-2019.
\end{IEEEbiography}

\vskip -2\baselineskip plus -1fil
\begin{IEEEbiography}[{\includegraphics[width=1in,height=1.25in,clip,keepaspectratio]{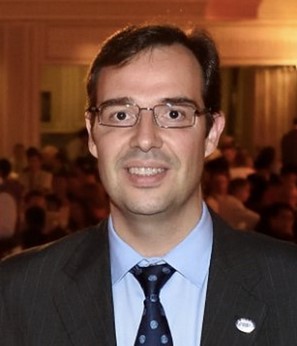}}]{Antonio Plaza} (M’05-SM’07-F’15) received the M.Sc. degree and the Ph.D. degree in computer engineering from the Hyperspectral Computing Laboratory, Department of Technology of Computers and Communications, University of Extremadura, C\'aceres, Spain, in 1999 and 2002, respectively. He is currently the Head of the Hyperspectral Computing Laboratory, Department of Technology of Computers and Communications, University of Extremadura. He has authored more than 600 publications, including over 200 JCR journal articles (over 160 in IEEE journals), 23 book chapters, and around 300 peer-reviewed conference proceeding papers. His research interests include hyperspectral data processing and parallel computing of remote sensing data.

Dr. Plaza was a member of the Editorial Board of the IEEE Geoscience and Remote Sensing Newsletter from 2011 to 2012 and the IEEE GEOSCIENCE AND REMOTE SENSING MAGAZINE in 2013. He was also a member of the Steering Committee of the IEEE JOURNAL OF SELECTED TOPICS IN APPLIED EARTH OBSERVATIONS AND REMOTE SENSING (JSTARS). He is also a fellow of IEEE for contributions to hyperspectral data processing and parallel computing of earth observation data. He received the recognition as a Best Reviewer of the IEEE GEOSCIENCE AND REMOTE SENSING LETTERS, in 2009, and the IEEE TRANSACTIONS ON GEOSCIENCE AND REMOTE SENSING, in 2010, for which he has served as an Associate Editor from 2007 to 2012. He was also a recipient of the Most Highly Cited Paper (2005–2010) in the Journal of Parallel and Distributed Computing, the 2013 Best Paper Award of the IEEE JOURNAL OF SELECTED TOPICS IN APPLIED EARTH OBSERVATIONS AND REMOTE SENSING (JSTARS), and the Best Column Award of the IEEE Signal Processing Magazine in 2015. He received Best Paper Awards at the IEEE International Conference on Space Technology and the IEEE Symposium on Signal Processing and Information Technology. He has served as the Director of Education Activities for the IEEE Geoscience and Remote Sensing Society (GRSS) from 2011 to 2012 and as the President of the Spanish Chapter of IEEE GRSS from 2012 to 2016. He has reviewed more than 500 manuscripts for over 50 different journals. He has served as the Editor-in-Chief of the IEEE TRANSACTIONS ON GEOSCIENCE AND REMOTE SENSING from 2013 to 2017. He has guestedited ten special issues on hyperspectral remote sensing for different journals. He is also an Associate Editor of IEEE ACCESS (received the recognition as an Outstanding Associate Editor of the journal in 2017). Additional information: http://www.umbc.edu/rssipl/people/aplaza
\end{IEEEbiography}

\vskip -2\baselineskip plus -1fil

\begin{IEEEbiography}[{\includegraphics[width=1in,height=1.25in,clip,keepaspectratio]{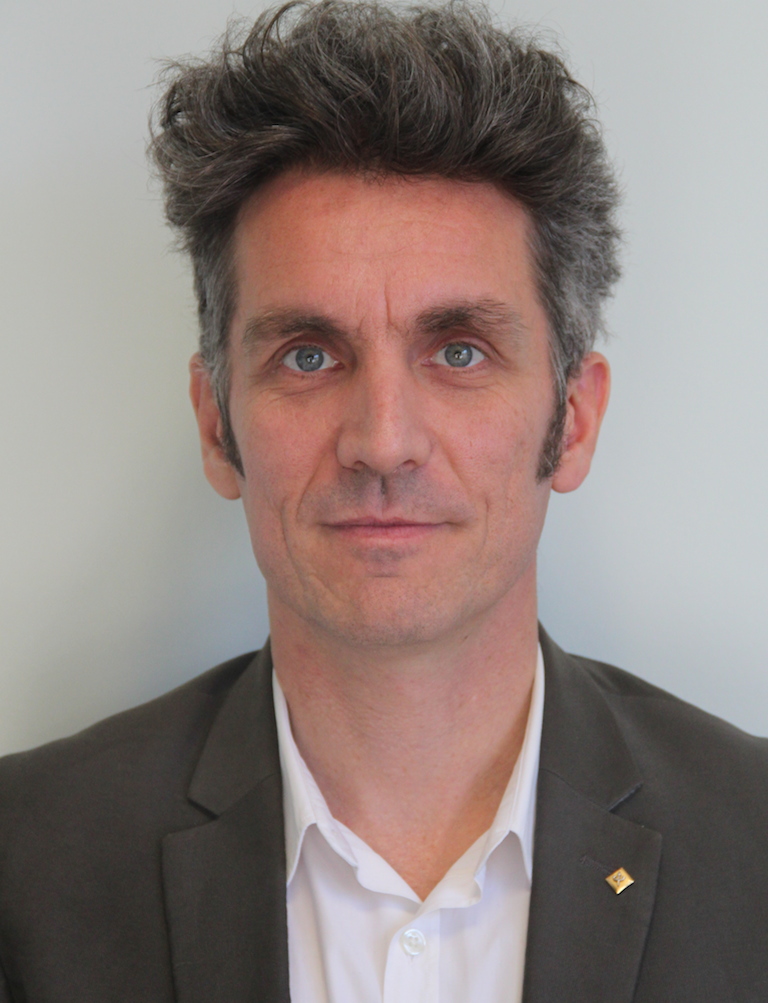}}]{Jocelyn Chanussot}
(M'04–SM'04-F'12) received the M.Sc. degree in electrical engineering from the Grenoble Institute of Technology (Grenoble INP), Grenoble, France, in 1995, and the Ph.D. degree from the Université de Savoie, Annecy, France, in 1998. Since 1999, he has been with Grenoble INP, where he is currently a Professor of signal and image processing. His research interests include image analysis, hyperspectral remote sensing, data fusion, machine learning and artificial intelligence. He has been a visiting scholar at Stanford University (USA), KTH (Sweden) and NUS (Singapore). Since 2013, he is an Adjunct Professor of the University of Iceland. In 2015-2017, he was a visiting professor at the University of California, Los Angeles (UCLA). He holds the AXA chair in remote sensing and is an Adjunct professor at the Chinese Academy of Sciences, Aerospace Information research Institute, Beijing.

Dr. Chanussot is the founding President of IEEE Geoscience and Remote Sensing French chapter (2007-2010) which received the 2010 IEEE GRS-S Chapter Excellence Award. He has received multiple outstanding paper awards. He was the Vice-President of the IEEE Geoscience and Remote Sensing Society, in charge of meetings and symposia (2017-2019). He was the General Chair of the first IEEE GRSS Workshop on Hyperspectral Image and Signal Processing, Evolution in Remote sensing (WHISPERS). He was the Chair (2009-2011) and  Cochair of the GRS Data Fusion Technical Committee (2005-2008). He was a member of the Machine Learning for Signal Processing Technical Committee of the IEEE Signal Processing Society (2006-2008) and the Program Chair of the IEEE International Workshop on Machine Learning for Signal Processing (2009). He is an Associate Editor for the IEEE Transactions on Geoscience and Remote Sensing, the IEEE Transactions on Image Processing and the Proceedings of the IEEE. He was the Editor-in-Chief of the IEEE Journal of Selected Topics in Applied Earth Observations and Remote Sensing (2011-2015). In 2014 he served as a Guest Editor for the IEEE Signal Processing Magazine. He is a Fellow of the IEEE, a member of the Institut Universitaire de France (2012-2017) and a Highly Cited Researcher (Clarivate Analytics/Thomson Reuters, 2018-2019).

\end{IEEEbiography}
\end{document}